\documentclass[twoside]{article}
\usepackage[accepted]{aistats2e}
\usepackage{graphicx}
\usepackage{amsfonts,amssymb,amsmath,amsthm}
\usepackage[numbers,square]{natbib}
\usepackage{hyperref}
\hypersetup{backref, colorlinks=true, citecolor=blue, linkcolor=blue}
\newcommand{\rref}[1]{\hyperref[#1]{\ref*{#1}}}
\usepackage{algorithm,algorithmic}
\newtheorem{theorem}{Theorem}[section]
\newtheorem{definition}[theorem]{Definition}
\newtheorem{lemma}[theorem]{Lemma}

\newtheorem{corollary}[theorem]{Corollary}

\newcommand{\F}{\mathcal{F}}
\newcommand{\E}{\mathbb{E}}
\renewcommand{\P}{\mathbb{P}}
\newcommand{\R}{\mathbb{R}}
\newcommand{\X}{\mathcal{X}}

\newcommand{\V}{\mathbb{V}}

\newcommand{\TV}[1]{\left|\left| #1 \right|\right|_{TV}}

\begin{document} 
\twocolumn[
\aistatstitle{Estimating $\beta$-mixing coefficients} 


\aistatsauthor{ Daniel J. McDonald \And Cosma Rohilla Shalizi \And
   Mark Schervish}
 \aistatsaddress{ Carnegie Mellon University
   \And Carnegie Mellon University\\ Santa Fe Institute \And
   Carnegie Mellon University} 
 ]

\begin{abstract} 
  The literature on statistical learning for time series assumes the asymptotic
  independence or ``mixing' of the data-generating process.  These mixing
  assumptions are never tested, nor are there methods for estimating mixing
  rates from data. We give an estimator for the $\beta$-mixing rate based on a
  single stationary sample path and show it is $L_1$-risk consistent.
\end{abstract} 

\section{Introduction}
\label{sec:introduction}

Relaxing the assumption of independence is an active area of research in the
statistics and machine learning literature. For time series, independence is
replaced by the asymptotic independence of events far apart in time, or
``mixing''.  Mixing conditions make the dependence of the future on the past
explicit, quantifying the decay in dependence as the future moves farther from
the past.  There are many definitions of mixing of varying strength with
matching dependence coefficients (see
\citep{Doukhan1994,dedecker2007weak,Bradley2005} for reviews), but most of the
results in the learning literature focus on $\beta$-mixing or absolute regularity.
Roughly speaking (see Definition \ref{defn:beta-mix} below for a precise
statement), the $\beta$-mixing coefficient at lag $a$ is the total variation
distance between the actual joint distribution of events separated by $a$ time
steps and the product of their marginal distributions, i.e., the $L_1$ distance
from independence.

Numerous results in the statistical machine learning literature rely
on knowledge of the $\beta$-mixing coefficients. As
\citet[p.~41]{Vidyasagar1997} notes, $\beta$-mixing is ``just right''
for 
the extension of IID results to dependent data, and so recent work has
consistently focused on it. \citet{Meir2000} derives generalization error
bounds for nonparametric methods based on model selection via structural risk
minimization. \citet{baraud2001adaptive} study the finite sample risk
performance of penalized least squares regression estimators under
$\beta$-mixing. \citet{LozanoKulkarni2006} examine regularized boosting
algorithms under absolute regularity and prove
consistency. \citet{KarandikarVidyasagar2009} consider ``probably approximately
correct'' learning algorithms, proving that PAC algorithms for IID inputs
remain PAC with $\beta$-mixing inputs under some mild
conditions. \citet{RalaivolaSzafranski2010} derive PAC bounds for ranking
statistics and classifiers using a decomposition of the dependency
graph. Finally, \citet{MohriRostamizadeh2010} derive stability bounds
for $\beta$-mixing inputs,
generalizing existing stability results for IID data.

All these results assume not just $\beta$-mixing, but known mixing
coefficients.  In particular, the risk bounds
in~\citep{Meir2000,MohriRostamizadeh2010}
and~\citep{RalaivolaSzafranski2010} are incalculable without knowledge of
the rates. This knowledge is \emph{never} available. Unless
researchers are willing to assume specific values for a sequence of
$\beta$-mixing coefficients, the results mentioned in the previous
paragraph are generally useless when confronted with data.To
illustrate this deficiency, consider 
Theorem 18 of~\citep{MohriRostamizadeh2010}:
\begin{theorem}[Briefly]
  Assume a learning algorithm is $\lambda$-stable. Then, for
  any sample of size $n$ drawn from a stationary $\beta$-mixing
  distribution, and $\epsilon>0$
  \[
  \P(|R-\widehat{R}|>\epsilon)\leq \Gamma(n,\lambda,\epsilon,a,b) +
  \beta(a)(\mu_n-1) 
  \]
  where $n=(a+b)\mu_n$, $\Gamma$ has a particular functional form, and
  $R-\widehat{R}$ is the difference between the true risk and the
  empirical risk.
\end{theorem}
Ideally, one could use this result for model selection or to control
the size of the generalization error of competing prediction
algorithms (support vector machines, support vector regression, and
kernel ridge regression are a few of the many algorithms known to
satisfy $\lambda$-stability). However the bound depends explicitly on
the mixing coefficient $\beta(a)$.  To
make matters worse, there are \emph{no} methods for estimating the
$\beta$-mixing coefficients. 
According
to \citet[p. 7]{Meir2000}, ``there is no efficient practical approach known at
this stage for estimation of mixing parameters.''  We begin to rectify
this problem by 
deriving the first method for estimating these coefficients. We prove
that our estimator is consistent for arbitrary $\beta$-mixing processes. 
In addition, we derive
rates of convergence for Markov approximations to these processes.

Application of statistical learning results to $\beta$-mixing data is
highly desirable in applied work. Many
common time series models are known to be $\beta$-mixing, 
and the rates of decay are known given the true
parameters of the process. Among the processes for which such knowledge is available
are ARMA models \cite{Mokkadem1988}, GARCH models
\cite{CarrascoChen2002}, and certain Markov processes --- see
\cite{Doukhan1994} for an overview of such results. To our knowledge,
only \citet{Nobel2006} approaches a solution to the problem of estimating
mixing rates by giving a
method to distinguish between different polynomial mixing rate regimes through
hypothesis testing.

We present the first method for estimating the $\beta$-mixing
coefficients for stationary time series data. Section
\ref{sec:estimation} defines the $\beta$-mixing 
coefficient and states our 
main results on convergence rates and consistency for our estimator.  Section
\ref{sec:conv-hist} gives an intermediate result on the $L_1$ convergence of
the histogram estimator with $\beta$-mixing inputs. Section \ref{sec:proof}
proves the main results from \S\ref{sec:estimation}. Section
\ref{sec:discussion} concludes and lays out some avenues for future research.

\section{Estimation of $\beta$-mixing}
\label{sec:estimation}


In this section, we present one of many equivalent definitions of absolute
regularity and state our main results, deferring proof to \S \ref{sec:proof}.

To fix notation, let $\mathbf{X}=\{X_t\}_{t=-\infty}^\infty$ be a sequence of
random variables where each $X_t$ is a measurable function from a probability
space $(\Omega,\F,\P)$ into a measurable space $\X$. A block of this random
sequence will be given by $\mathbf{X}_i^j \equiv \{X_t\}_{t=i}^j$ where $i$ and
$j$ are integers, and may be infinite.  We use similar notation for
the sigma fields generated by these blocks and their joint
distributions.  In particular, $\sigma_i^j$ will denote the sigma
field generated by $\mathbf{X}_i^j$, and the joint distribution of
$\mathbf{X}_i^j$ will be denoted $\P_i^j$.

\subsection{Definitions}
\label{sec:definitions}

There are many equivalent definitions of $\beta$-mixing
(see for instance~\citep{Doukhan1994}, or~\citep{Bradley2005} as well
as~\citet{Meir2000} or~\citet{Yu1994}), however the most intuitive is that given
in~\citet{Doukhan1994}.
\begin{definition}[$\beta$-mixing]
\label{defn:beta-mix}
For each positive integer $a$, the the {\em coefficient of absolute
  regularity}, or {\em   $\beta$-mixing coefficient}, $\beta(a)$, is 
\begin{equation}
  \label{eq:13}
  \beta(a) \equiv \sup_t \TV{\P_{-\infty}^t \otimes \P_{t+a}^\infty - \P_{t,a}}
\end{equation}
where $|| \cdot ||_{TV}$ is the total variation norm, and $\P_{t,a}$
is the joint distribution of $(\mathbf{X}_{-\infty}^t,\mathbf{X}_{t+a}^\infty)$. A
stochastic process is said to be {\em absolutely regular}, or
{\em $\beta$-mixing}, if $\beta(a) \rightarrow 0$ as $a\rightarrow\infty$.
\end{definition}

Loosely speaking, Definition~\ref{defn:beta-mix} says that the 
coefficient $\beta(a)$ measures the total variation distance between the joint
distribution of random variables seaparted by $a$ time units and a
distribution under which random variables separated by $a$ time units are independent.
The supremum over $t$ is unnecessary for stationary
random processes $\mathbf{X}$ which is the only case we consider here.
\begin{definition}[Stationarity]\label{def:stationary}
  A sequence of random variables $\mathbf{X}$ is {\em stationary}
  when all its finite-dimensional distributions are invariant over time: for
  all $t$ and all non-negative integers $i$ and $j$, the random vectors
  $\mathbf{X}_t^{t+i}$ and $\mathbf{X}_{t+j}^{t+i+j}$ have the same
  distribution.
\end{definition}

Our main result requires the method of blocking used by~\citet{Yu1993,Yu1994}.
The purpose is to transform a sequence of dependent variables into subsequence
of nearly IID ones.  Consider a sample $\mathbf{X}_1^{n}$ from a stationary
$\beta$-mixing sequence with density $f$. Let $m_n$ and $\mu_n$ be non-negative
integers such that $2m_n\mu_n=n$. Now divide $\mathbf{X}_1^{n}$ into $2\mu_n$
blocks of each length $m_n$. Identify the blocks as follows:
\begin{align*}
  U_j &= \{X_i: 2(j-1)m_n + 1 \leq i \leq (2j-1)m_n\},\\
  V_j &= \{X_i : (2j-1)m_n + 1 \leq i \leq 2jm_n\}.
\end{align*}
Let $\mathbf{U}$ be the entire sequence of odd blocks $U_j$, and let
$\mathbf{V}$ be the sequence of even blocks $V_j$. Finally, let
$\mathbf{U}'$ be a sequence of blocks which are independent of $\mathbf{X}_1^n$
but such that each block has the same distribution as a block from the
original sequence:
\begin{align}
  \label{eq:2}
  U'_j \overset{D}{=} U_j \overset{D}{=} U_1.
\end{align}
The blocks $\mathbf{U}'$ are now an IID block sequence, so standard results
apply.  (See \citep{Yu1994} for a more rigorous analysis of blocking.) With
this structure, we can state our main result.

\subsection{Results}
\label{sec:results}

Our main result emerges in two stages. First, we recognize that the
distribution of a finite sample depends only on finite-dimensional
distributions.  This leads to an estimator of a finite-dimensional
version of $\beta(a)$.  Next, we let the finite-dimension increase to
infinity with the size of the observed sample.

For positive integers $t$, $d$, and $a$, define
\begin{equation}
  \label{eq:2a}
  \beta^d(a)=\TV{\P_{t-d+1}^t \otimes \P_{t+a}^{t+a+d-1}-
    \P_{t,a,d}},
\end{equation}
where $\P_{t,a,d}$ is the joint distribution of $(\mathbf{X}_{t+d+1}^t,
\mathbf{X}_{t+a}^{t+a+d-1})$.  Also, let $\widehat{f}^d$ be the
$d$-dimensional histogram estimator of the joint density of $d$
consecutive observations, and let $\widehat{f}_a^{2d}$ be the
$2d$-dimensional histogram estimator of the joint density of two sets
of $d$ consecutive observations separated by $a$ time points. 

We construct an estimator of $\beta^d(a)$ based on these two
histograms.\footnote{While it is clearly possible to replace
  histograms with other choices of density estimators (most notably
  KDEs), histograms in this case are more convenient theoretically
  and computationally. See \S\ref{sec:discussion} for more details.} Define
\begin{equation}
  \label{eq:estimator}
  \widehat{\beta}^{d}(a) = \frac{1}{2}\int\left|\widehat{f}_a^{2d} -
    \widehat{f}^d\otimes\widehat{f}^d\right|
\end{equation}
We show that, by allowing $d=d_n$ to grow with $n$, this estimator will
converge on $\beta(a)$. This can be seen most clearly by bounding
the $L_1$-risk of the estimator with its estimation and
approximation errors:
\[
|\widehat{\beta}^{d_n} - \beta(a)| \leq |\widehat{\beta}^{d_n} -
\beta^{d_n}| + |\beta^{d_n} -  \beta(a)|.
\]
The first term is the error of estimating $\beta^d(a)$ with a random
sample of data. The second term is the non-stochastic error induced by
approximating the infinite dimensional coefficient, $\beta(a)$, with
its $d$-dimensional counterpart, $\beta^d(a)$.

Our first theorem in this section
establishes consistency of $\widehat{\beta}^{d_n}(a)$ as an estimator
of $\beta(a)$ for all
$\beta$-mixing processes provided $d_n$ increases at an appropriate
rate. Theorem~\ref{thm:main} gives finite sample bounds on the
estimation error while some measure theoretic arguments contained in
\S\ref{sec:proof} show that the approximation error must go to zero as $d_n
\rightarrow \infty$.
\begin{theorem}
  \label{thm:two}
  Let $\mathbf{X}_1^n$ be a sample from an arbitrary $\beta$-mixing process. Let
  $d_n=O(\exp\{W(\log n)\})$ where $W$ is the Lambert $W$
  function.\footnote{The Lambert $W$ function is defined as the
    (multivalued) inverse
    of $f(w) = w\exp\{w\}$. Thus, $O(\exp\{W(\log n)\})$ is bigger
    than $O(\log\log n)$ but smaller than $O(\log n)$. See for
    example~\citet{CorlessGonnet1996}. }  Then
  $\widehat{\beta}^{d_n}(a)\xrightarrow{P}\beta(a)$ as
  $n\rightarrow\infty$.  
\end{theorem}

A finite sample bound for the approximation error is the first step to
establishing consistency for $\widehat{\beta}^{d_n}$. This result
gives convergence rates for estimation of the finite dimensional
mixing coefficient $\beta^d(a)$ and also for Markov
processes of known order $d$, since in this case, $\beta^d(a) =
\beta(a)$. 
\begin{theorem}
  \label{thm:main}
  Consider a sample $\mathbf{X}_1^n$ from a stationary
  $\beta$-mixing process. Let $\mu_n$ and $m_n$
  be positive integers such that $2\mu_n m_n=n$ and $\mu_n\geq d>0$.  Then
\begin{align*}
\lefteqn{\P(|\widehat{\beta}^d(a) - \beta^d(a)|>\epsilon)}\\
    & \leq
    2\exp\left\{-\frac{\mu_n \epsilon_1^2}{2}\right\}
    + 2\exp\left\{-\frac{\mu_n \epsilon_2^2}{2}\right\}\\
    &+ 4(\mu_n-1)\beta(m_n),
\end{align*} 
  where $\epsilon_1 = \epsilon/2-\E\left[\int|\widehat{f}^d - f^d|\right]$ and
  $\epsilon_2 = \epsilon - \E\left[\int|\widehat{f}_a^{2d} - f_a^{2d}|\right]$.
\end{theorem}

Consistency of the estimator $\widehat{\beta}^d(a)$ is guaranteed only
for certain choices of $m_n$ and $\mu_n$. Clearly
$\mu_n\rightarrow\infty$ and $\mu_n\beta(m_n)\rightarrow 0$ as
$n\rightarrow\infty$ are necessary 
conditions. Consistency also requires convergence of the histogram
estimators to the target densities. We leave the proof of this theorem for
section~\ref{sec:proof}. As an example to show that this bound can go
to zero with proper choices of $m_n$ and $\mu_n$, the following corollary proves consistency for
first order Markov processes. Consistency of the estimator for higher order Markov
processes can be proven similarly. These processes are algebraically
$\beta$-mixing as shown in e.g.~\citet{NummelinTuominen1982}.
\begin{corollary}
  \label{cor:cor1}
  Let $\mathbf{X}_1^n$ be a sample from a first order Markov process
  with $\beta(a) = \beta^1(a)= O(a^{-r})$. Then under the conditions of
  Theorem~\ref{thm:main}, $\widehat{\beta}^1(a)\xrightarrow{P}\beta(a)$.
\end{corollary}
\begin{proof}
  Recall that $n = 2\mu_n m_n$. Then,
  \begin{align*}
   4(\mu_n-1)\beta(m_n) &= 4\mu_n\beta(m_n) + 4\beta(m_n)\\
   &= K_1\frac{n}{m_n}m_n^{-r} + K_2m_n^{-r}\\
   &\rightarrow 0
 \end{align*}
 if $m_n< n^{1/(1+r)}$ for constants $K_1$ and $K_2$. In this case, we have that the
 exponential terms are less than
 \[
 \exp\left\{ -K_3 \frac{n \epsilon_j^2}{n^{1/(1+r)}} \right\} =
 \exp\left\{ -K_3 n^{r/(1+r)} \epsilon_j^2\right\},
 \]
 for $j=1,2$ and a constant $K_3$. Therefore, both exponential terms
 go to 0 as $n\rightarrow\infty$.
\end{proof}

Proving Theorem~\ref{thm:main} requires showing the $L_1$ convergence of the
histogram density estimator with $\beta$-mixing data. We do this in the next
section.

\section{$L_1$ convergence of histograms}
\label{sec:conv-hist}

Convergence of density estimators is thoroughly studied in the statistics and
machine learning literature. Early papers on the $L_\infty$ convergence of
kernel density estimators (KDEs) include
\cite{Woodroofe1967,BickelRosenblatt1973,Silverman1978};
\citet{FreedmanDiaconis1981} look specifically at histogram estimators, and
\citet{Yu1993} considered the $L_\infty$ convergence of KDEs for $\beta$-mixing
data and shows that the optimal IID rates can be attained.
\citet{DevroyeGyorfi1985} argue that $L_1$ is a more appropriate metric for
studying density estimation, and \citet{Tran1989} proves $L_1$ consistency of
KDEs under $\alpha$- and $\beta$-mixing. As far as we are aware, ours is the
first proof of $L_1$ convergence for histograms under
$\beta$-mixing. 

Additionally, the dimensionality of the target density
is analogous to the order of the Markov approximation. Therefore, the
convergence rates we give are asymptotic in the bandwidth $h_n$ which
shrinks as $n$ increases, but also in the dimension $d$ which
increases with $n$. Even under these asymptotics, histogram estimation
in this sense is not a high dimensional problem. The dimension of the
target density considered here is on the order of $\exp\{W(\log n)\}$,
a rate somewhere between $\log n$ and $\log\log n$.

\begin{theorem}
  \label{thm:one}
  If $\widehat{f}$ is the histogram
  estimator based on a (possibly vector valued) sample $\mathbf{X}_1^n$ from a $\beta$-mixing sequence with
  stationary density $f$, then for all 
  $\epsilon>\E\left[\int |\widehat{f}-f|\right]$, 
  \begin{align}
    \label{eq:5}
    \P\left(\int |\widehat{f}-f| > \epsilon\right) &\leq
    2\exp\left\{-\frac{\mu_n\epsilon_1^{2}}{2}\right\}\nonumber\\ 
    &+ 2(\mu_n-1)\beta(m_n)
  \end{align}
  where $\epsilon_1 = \epsilon-\E\left[\int |\widehat{f}-f|\right]$.
\end{theorem}
To prove this result, we use the blocking method of~\citet{Yu1994} to
transform the dependent $\beta$-mixing into a sequence of nearly
independent blocks. We then apply McDiarmid's inequality to the blocks
to derive asymptotics in the bandwidth of the histogram as well as the
dimension of the target density. For completeness, we state Yu's
blocking result and McDiarmid's inequality before proving the doubly
asymptotic histogram convergence for IID data. Combining these lemmas
allows us to derive rates of convergence for histograms based on
$\beta$-mixing inputs.

\begin{lemma}[Lemma 4.1 in~\citet{Yu1994}]
  \label{lem:yu}
  Let $\phi$ be a measurable function with respect to the block sequence
  $\mathbf{U}$ uniformly bounded by $M$. Then,
  \begin{equation}
    \label{eq:8}
    |\E[\phi] - \tilde{\E}[\phi]| \leq M\beta(m_n)(\mu_n-1),
  \end{equation}
  where the first expectation is with respect to the dependent block
  sequence, $\mathbf{U}$, and
  $\tilde{\E}$ is with respect to the independent sequence, $\mathbf{U}'$.
\end{lemma}
This lemma essentially gives a method of applying IID results to
$\beta$-mixing data. Because the dependence decays as we increase the
separation between blocks, widely spaced blocks are nearly independent
of each other. In particular, the difference between expectations over
these nearly independent blocks and expectations over blocks which are
actually independent can be controlled by the $\beta$-mixing
coefficient. 

\begin{lemma}[McDiarmid Inequality~\citep{McDiarmid1989}]
  \label{lem:mcdiarmid}
  Let $X_1,\ldots,X_n$ be independent random variables, with $X_i$
  taking values in a set $A_i$ for each $i$. Suppose that the
  measurable function $f:\prod A_i \rightarrow \R$ satisfies
  \[
  |f(\mathbf{x})-f(\mathbf{x}')| \leq c_i
  \]
  whenever the vectors $\mathbf{x}$ and $\mathbf{x}'$ differ only in
  the $i^{th}$ coordinate. Then for any $\epsilon>0$,
  \[
  \P(f - \E f > \epsilon) \leq \exp\left\{-\frac{2\epsilon^2}{\sum
      c_i^2} \right\}.
  \]
\end{lemma}

\begin{lemma}
  \label{lem:three}
  For an IID sample $X_1,\ldots,X_n$ from some density $f$ on $\R^d$,
  \begin{align}
    \label{eq:6}
    \E \int |\widehat{f}-\E\widehat{f}|dx  &= O\left(1/\sqrt{nh_n^d}\right)\\
    \int |\E\widehat{f} - f|dx &= O(dh_n)+O(d^2h_n^2),
  \end{align}
  where $\widehat{f}$ is the histogram estimate using a grid with
  sides of length $h_n$.
\end{lemma}
\begin{proof}[Proof of Lemma~\ref{lem:three}]
  Let $p_j$ be the probability of falling into the $j^{th}$ bin
  $B_j$. Then,
  \begin{align*}
    \E \int |\widehat{f} - \E\widehat{f}| &= h_n^d \sum_{j=1}^J
    \E\left| \frac{1}{nh_n^d} \sum_{i=1}^n I(X_i \in B_j) -
      \frac{p_j}{h^d} \right|\\
    &\leq h^d_n \sum_{j=1}^J \frac{1}{nh_n^d} \sqrt{\V\left[\sum_{i=1}^n I(X_i \in B_j)\right]}\\
    &= h^d_n \sum_{j=1}^J \frac{1}{nh_n^d} \sqrt{np_j(1-p_j)}\\
    &= \frac{1}{\sqrt{n}}\sum_{j=1}^J \sqrt{p_j(1-p_j)}\\
    &= O(n^{-1/2})O(h_n^{-d/2}) = O\left(1/\sqrt{nh_n^d}\right).
  \end{align*}
  For the second claim, consider the bin $B_j$ centered at 
  $\mathbf{c}$. Let $I$ be the union of all bins $B_j$. Assume the following:
  \begin{enumerate}
  \item $f \in L_2$ and $f$ is absolutely continuous on $I$, with
    a.e.~partial derivatives $f_i=\frac{\partial}{\partial y_i} f(\mathbf{y})$
  \item $f_i \in L_2$ and $f_i$ is absolutely continuous on $I$, with
    a.e.~partial derivatives $f_{ik}=\frac{\partial}{\partial y_k} f_i(\mathbf{y})$
  \item $f_{ik} \in L_2$ for all $i,k$.
  \end{enumerate}
  Using a Taylor expansion 
  \begin{align*}
   f(\mathbf{x}) &= f(\mathbf{c}) + \sum_{i=1}^d
    (x_i-c_i)f_i(\mathbf{c})+ O(d^2h_n^2),
  \end{align*}
  where
  $f_i(\mathbf{y}) = \frac{\partial}{\partial y_i} f(\mathbf{y})$. Therefore, $p_j$
  is given by 
  \begin{equation*}p_j = \int_{B_j} f(x)dx = h_n^df(c) +
    O(d^2h_n^{d+2})
  \end{equation*}
  since the integral of the second term over the bin
  is zero. This means that for the $j^{th}$ bin,
  \begin{align*}
    \E\widehat{f}_n(x)-f(x) &=
    \frac{p_j}{h_n^d} - f(x)\\ & = -\sum_{i=1}^d
    (x_i-c_i)f_i(\mathbf{c}) + O(d^2h_n^2).
  \end{align*}
  Therefore,
  \begin{align*}
    \lefteqn{\displaystyle \int_{B_j} \left| \E\widehat{f}_n(x)-f(x)
  \right|}\\
    &= \int_{B_j}
    \left| -\sum_{i=1}^d (x_i - c_i)f_i(\mathbf{c}) + O(d^2h_n^2)
    \right|\\
    &\leq \int_{B_j} \left| -\sum_{i=1}^d (x_i - c_i)f_i(\mathbf{c})
    \right| + \int_{B_j} O(d^2h^2)\\
   &= \int_{B_j} \left| \sum_{i=1}^d (x_i - c_i) f_i(\mathbf{c})
    \right| + O(d^2h_n^{2+d})\\
    &=  O(dh_n^{d+1}) + O(d^2h_n^{2+d})
  \end{align*}
Since each bin is bounded, we can sum over all $J$ bins. The number of
bins is $J=h_n^{-d}$ by definition, so 
\begin{align*}
  \lefteqn{\int|\E\widehat{f}_n(x)-f(x)|dx}& \\
  &= O(h_n^{-d})\left( O(dh_n^{d+1}) +  O(d^2h_n^{2+d})\right)\\
  & = O(dh_n) + O(d^2h_n^2).
\end{align*}
\end{proof}

We can now prove the main result of this section.
\begin{proof}[Proof of Theorem~\ref{thm:one}]
  Let $g$ be the $L_1$ loss of the histogram estimator, $g=\int
  |f-\widehat{f}_n|$. Here $\widehat{f}_n(x) =
  \frac{1}{nh_n^d}\sum_{i=1}^n I(X_i \in B_j(x))$ where $B_j(x)$ is
  the bin containing $x$. Let $\widehat{f}_{\mathbf{U}}$,
  $\widehat{f}_{\mathbf{V}}$, and $\widehat{f}_{\mathbf{U}'}$ be
  histograms based on the block sequences $\mathbf{U}$, $\mathbf{V}$,
  and $\mathbf{U}'$ respectively. Clearly $\widehat{f}_n =
  \frac{1}{2}(\widehat{f}_{\mathbf{U}} + \widehat{f}_{\mathbf{V}}).$ Now,
  \begin{align*}
    \P(g>\epsilon) &= \P\left(\int |f-\widehat{f}_n| >
      \epsilon\right)\\
    &= \P\left(\int \left|\frac{f-\widehat{f}_{\mathbf{U}}}{2} +
          \frac{f-\widehat{f}_{\mathbf{V}}}{2} \right| >
      \epsilon\right)\\
    &\leq \P\left( \frac{1}{2} \int |f-\widehat{f}_{\mathbf{U}}| +
      \frac{1}{2} \int|f-\widehat{f}_{\mathbf{V}}| >
      \epsilon\right)\\
    &= \P(g_{\mathbf{U}} + g_{\mathbf{V}} > 2\epsilon)\\
    &\leq \P(g_{\mathbf{U}} >\epsilon) + \P(g_{\mathbf{V}}>\epsilon)\\
    &= 2\P(g_{\mathbf{U}} - \E[g_{\mathbf{U}}] > \epsilon-\E[g_{\mathbf{U}}])\\
    &= 2\P(g_{\mathbf{U}} - \E[g_{\mathbf{U}'}] >
    \epsilon-\E[g_{\mathbf{U}'}])\\
    &= 2\P(g_{\mathbf{U}} - \E[g_{\mathbf{U}'}] > \epsilon_1),
  \end{align*}
  where $\epsilon_1 = \epsilon-\E[g_{\mathbf{U}'}]$. Here,
  \[
  \E[g_{\mathbf{U}'}] \leq \tilde{\E} \int
  |\widehat{f}_{\mathbf{U'}}-\tilde{\E}\widehat{f}_{\mathbf{U'}}|dx + 
  \int |\tilde{\E}\widehat{f}_{\mathbf{U'}} - f|dx, 
  \]
  so by Lemma~\ref{lem:three}, as long as for $\mu_n\rightarrow\infty$, $h_n
  \downarrow 0$ and $\mu_nh_n^d \rightarrow \infty$, then for all
  $\epsilon$ there exists $n_0(\epsilon)$ such that for all $n>n_0(\epsilon)$,
  $\epsilon>\E[g]=\E[g_{\mathbf{U}'}]$. Now
  applying Lemma~\ref{lem:yu} to the expectation of the indicator of
  the event $\{g_{\mathbf{U}} - \E[g_{\mathbf{U}'}] > \epsilon_1\}$
  gives 
  \begin{align*}
    2\P(g_{\mathbf{U}} - \E[g_{\mathbf{U}'}] > \epsilon_1) &\leq
    2\P(g_{\mathbf{U}'} - \E[g_{\mathbf{U}'}] > \epsilon_1)\\ &+ 2(\mu_n-1)\beta(m_n)
  \end{align*}
  where the probability on the right is for the $\sigma$-field
  generated by the independent block sequence $\mathbf{U}'$. Since
  these blocks are independent, showing that $g_{\mathbf{U}'}$
  satisfies the bounded differences requirement allows for the
  application of McDiarmid's inequality~\ref{lem:mcdiarmid} to the blocks. For any two block
  sequences $u'_1,\ldots,u'_{\mu_n}$ and $\bar{u}'_1,\ldots,\bar{u}'_{\mu_n}$ with
  $u'_\ell=\bar{u}'_\ell$ for all $\ell\neq j$, then
  \begin{align*}
    &\left| g_{\mathbf{U}'}(u'_1,\ldots,u'_{\mu_n}) -
      g_{\mathbf{U}'}(\bar{u}'_1,\ldots,\bar{u}'_{\mu_n}) \right|\\ & = 
    \left|\int |\widehat{f}(y;u'_1,\ldots,u'_{\mu_n})-f(y)|dy\right.\\ &-
    \left.\int|\widehat{f}(y;\bar{u}'_1,\ldots,\bar{u}'_{\mu_n}) - f(y)|dy\right|\\
    &\leq \int |\widehat{f}(y;u'_1,\ldots,u'_{\mu_n}) -
    \widehat{f}(y;\bar{u}'_1,\ldots,\bar{u}'_{\mu_n})|dy\\
    &= \frac{2}{\mu_nh_n^d} h_n^d = \frac{2}{\mu_n}.
  \end{align*}
  Therefore,
  \begin{align*}
     \P(g>\epsilon) &\leq 2\P(g_{\mathbf{U}'} - \E[g_{\mathbf{U}'}] >
     \epsilon_1) + 2(\mu_n-1)\beta(m_n) \\
     & \leq 2\exp\left\{ -\frac{\mu_n\epsilon_1^2}{2}\right\} +
     2(\mu_n-1)\beta(m_n). 
  \end{align*}
\end{proof}

\section{Proofs}
\label{sec:proof}

The proof of Theorem~\ref{thm:main} relies on the triangle inequality and the
relationship between total variation distance and the $L_1$ distance
between densities.

\begin{proof}[Proof of Theorem~\ref{thm:main}]
  For any probability measures $\nu$ and $\lambda$ defined on the same
  probability space with associated densities $f_\nu$ and $f_\lambda$
  with respect to some dominating measure $\pi$,
  \[
  \TV{\nu-\lambda} = \frac{1}{2}\int |f_\nu - f_\lambda|d(\pi).
  \]
  Let $P$ be the $d$-dimensional stationary distribution of the
  $d^{th}$ order Markov
  process, i.e.~$P=\P_{t-d+1}^{t}=\P_{t+a}^{t+a+d-1}$ in the notation
  of equation~\ref{eq:2a}. Let $\P_{a,d}$ be the
  joint distribution of the bivariate random process created by the
  initial process and itself separated by $a$ time steps. By the
  triangle inequality, we can upper bound $\beta^d(a)$ for any $d=d_n$. Let
  $\widehat{P}$ and $\widehat{\P}_{a,d}$ be the distributions
  associated with histogram estimators $\widehat{f}^d$ and
  $\widehat{f}_a^{2d}$ respectively. Then,
  \begin{align*}
    \beta^d(a) &= \TV{P \otimes P - \P_{a,d}}\\
    &=\left|\left|P\otimes P -\widehat{P}\otimes\widehat{P} +
        \widehat{P}\otimes\widehat{P}\right.\right.\\
    &- \left.\left.\widehat{\P}_{a,d} + \widehat{\P}_{a,d}
        - \P_{a,d}\right|\right|_{TV}\\
    &\leq \TV{P\otimes P -\widehat{P}\otimes\widehat{P}} +
    \TV{\widehat{P}\otimes\widehat{P} - \widehat{\P}_{a,d}}\\
    &+ \TV{ \widehat{\P}_{a,d} - \P_{a,d}}\\
    &\leq 2\TV{P - \widehat{P}} +
    \TV{\widehat{P}\otimes\widehat{P} - \widehat{\P}_{a,d}}\\
    &+ \TV{ \widehat{\P}_{a,d} - \P_{a,d}}\\
    &= \int |f^d - \widehat{f}^d| +\frac{1}{2}\int |\widehat{f}^d\otimes
    \widehat{f}^d - \widehat{f}_a^{2d}|\\
    &+ \frac{1}{2}\int|f_a^{2d} -\widehat{f}_a^{2d}|  
  \end{align*}
  where $\frac{1}{2}\int |\widehat{f}^d\otimes
    \widehat{f}^d - \widehat{f}_a^{2d}|$ is our estimator $\widehat{\beta}^d(a)$ and
  the remaining terms are the $L_1$ distance between a
  density estimator and the target density. Thus,
  \[
  \beta^d(a) - \widehat{\beta}^d(a) \leq \int |f^d - \widehat{f}^d| +
  \frac{1}{2}\int|f_a^{2d} -\widehat{f}_a^{2d}| .
  \]
  A similar argument starting from $\beta^d(a)=\TV{P \otimes P -
    \P_{a,d}}$ shows that  
  \[
  \beta^d(a) - \widehat{\beta}^d(a) \geq -\int |f^d - \widehat{f}^d| -
  \frac{1}{2}\int|f_a^{2d} -\widehat{f}_a^{2d}|,
  \]
  so we have that
  \[
  \left|\beta^d(a) - \widehat{\beta}^d(a)\right| \leq \int |f^d - \widehat{f}^d| +
  \frac{1}{2}\int|f_a^{2d} -\widehat{f}_a^{2d}| .
  \]
  Therefore, 
  \begin{align*}
\lefteqn{\P\left(\left|\beta^d(a) - \widehat{\beta}^d(a)\right| >
      \epsilon\right)}\\
    & \leq \P\left( \int |f^d - \widehat{f}^d| +
      \frac{1}{2}\int|f_a^{2d} -\widehat{f}_a^{2d}|>\epsilon\right)\\
    & \leq \P\left( \int |f^d - \widehat{f}^d|>\frac{\epsilon}{2}\right) + \P\left(
      \frac{1}{2}\int|f_a^{2d} -\widehat{f}_a^{2d}|>\frac{\epsilon}{2}\right)\\
    &\leq 2\exp\left\{-\frac{\mu_n\epsilon_1^2}{2}\right\} +
    2\exp\left\{-\frac{\mu_n\epsilon_2^2}{2}\right\}\\ 
    &+ 4(\mu_n-1)\beta(m_n),
  \end{align*}
  where $\epsilon_1 = \epsilon/2-\E\left[\int |\widehat{f}^d - f^d|\right]$
  and $\epsilon_2 = \epsilon-\E\left[\int |\widehat{f}_a^{2d} - f_a^{2d}|\right]$.
\end{proof}

The proof of Theorem~\ref{thm:two} requires two steps which are given
in the following Lemmas. The first specifies the histogram
bandwidth $h_n$ and the rate at which
$d_n$ (the dimensionality of the target density) goes to infinity. If
the dimensionality of the target density were fixed, we 
could achieve rates of convergence similar to those for histograms
based on IID inputs. However, we wish to allow the dimensionality
to grow with $n$, so the rates are much slower as shown in the
following lemma.
\begin{lemma}
  \label{lem:two1}
  For the histogram estimator in Lemma~\ref{lem:three}, let
  \begin{align*}
    d_n &\sim \exp\{W(\log n)\},\\
    h_n &\sim n^{-k_n},
  \end{align*}
  with 
  \[
  k_n = \frac{ W(\log n) + \frac{1}{2} \log n}{\log n \left( \frac{1}{2}
      \exp\{ W(\log n)\} + 1\right)}.
  \]
These choices lead to the optimal rate of convergence.
\end{lemma}
\begin{proof}
  Let $h_n=n^{-k_n}$ for some $k_n$ to be determined. Then we
  want $n^{-1/2}h_n^{-d_n/2} = n^{(k_nd_n-1)/2}\rightarrow 0$, $d_nh_n =
  d_nn^{-k}\rightarrow 0$, and $d_n^2h_n^2 = d_n^2n^{-2k}\rightarrow 0$ all as
  $n\rightarrow\infty$. Call these $A$, $B$, and $C$. Taking $A$ and $B$
  first gives
  \begin{align}
    n^{(k_nd_n-1)/2} &\sim d_nn^{-k_n}\nonumber\\
    \Rightarrow \frac{1}{2} (k_nd_n-1) \log n &\sim \log d_n-k_n\log n \nonumber\\
    \Rightarrow k_n\log n \left(\frac{1}{2}d_n+1\right) &\sim \log d_n +
    \frac{1}{2} \log n\nonumber\\
    \label{eq:3}
    \Rightarrow k_n &\sim \frac{\log d_n + \frac{1}{2}\log n}{\log n
      \left(\frac{1}{2} d_n + 1\right)}.
  \end{align}
  Similarly, combining $A$ and $C$ gives
  \begin{equation}
    \label{eq:4}
    k_n \sim \frac{2\log d_n + \frac{1}{2}\log n}{\log n
      \left(\frac{1}{2} d_n + 2\right)}.
  \end{equation}
  Equating~(\ref{eq:3}) and~(\ref{eq:4}) and solving for $d_n$ gives
  \[
  \Rightarrow d_n \sim \exp\left\{ W(\log n)\right\}
  \]
  where $W(\cdot)$ is the Lambert $W$ function. Plugging back
  into~(\ref{eq:3}) gives that
  \[
  h_n = n^{-k_n}
  \]
  where
 \[ 
 k_n = \frac{
   W(\log n) + \frac{1}{2}\log n} { \log n \left( \frac{1}{2}
     \exp\left\{ W(\log n)\right\} + 1\right)}.
 \]
\end{proof}
It is also necessary to show that as $d$ grows, $\beta^d(a)\rightarrow
\beta(a)$. We now prove this result.
\begin{lemma}
  \label{lem:two2}
  $\beta^d(a)$ converges to $\beta(a)$ as $d\rightarrow\infty$.
\end{lemma}
\begin{proof}
  By stationarity, the supremum over $t$ is unnecessary in
  Definition~\ref{defn:beta-mix}, so without loss of generality, let $t=0$. 
  Let $\P_{-\infty}^0$ be the distribution on
  $\sigma_{-\infty}^0=\sigma(\ldots,X_{-1},X_0)$,
  and let $\P_a^\infty$ be the distribution on
  $\sigma_{a+1}^\infty=\sigma(X_{a+1},X_{a+2},\ldots)$. Let $\P_a$ be the
  distribution on 
  $\sigma=\sigma_{-\infty}^0\otimes\sigma_{a+1}^\infty$ (the product sigma-field).  Then we
  can rewrite Definition~\ref{defn:beta-mix} using this notation as
  \begin{equation*}
    \beta(a)=\sup_{C\in\sigma}|\P_a(C)-[\P_{-\infty}^0\otimes \P_a^\infty](C)|.
  \end{equation*}
  Let $\sigma_{-d+1}^0$ and $\sigma_{a+1}^{a+d}$ be the sub-$\sigma$-fields of
  $\sigma_{-\infty}^0$ and $\sigma_{a+1}^\infty$ consisting of the $d$-dimensional cylinder
  sets for the $d$ dimensions closest together.  Let $\sigma^d$ be the
  product $\sigma$-field of these two.  Then we can rewrite
  $\beta^d(a)$ as 
  \begin{equation}
    \beta^d(a)=\sup_{C\in\sigma^d}||\P_a(C)-[\P_{-\infty}^0\otimes \P_a^\infty](C)|.  
    \label{eq:11}
  \end{equation}
  As such $\beta^d(a)\leq\beta(a)$ for all $a$ and $d$. We can
  rewrite~(\ref{eq:11}) in terms of finite-dimensional marginals: 
  \begin{equation*}
    \beta^d(a)=\sup_{C\in\sigma^d}|\P_{a,d}(C)-[\P_{-d}^0\otimes \P_a^{a+d}](C)|,
  \end{equation*}
  where $\P_{a,d}$ is the restriction of $\P$ to
  $\sigma(X_{-d},\ldots,X_0,X_a,\ldots,X_{a+d})$. Because of
  the nested nature of these sigma-fields, we have
  \begin{equation*}
    \beta^{d_1}(a) \leq \beta^{d_2}(a) \leq \beta(a)
  \end{equation*}
  for all finite $d_1 \leq d_2$. Therefore, for fixed $a$,
  $\{\beta^d(a)\}_{d=1}^\infty$ is a monotone increasing sequence
  which is bounded above, and it converges
  to some limit $L\leq \beta(a)$. To show that $L=\beta(a)$ requires
  some additional steps.

   Let $R=\P_a-[\P_{-\infty}^0\otimes \P_a^\infty]$, which is a signed measure on $\sigma$. Let $R^d
  = \P_{a,d}-[\P_{-d}^0\otimes \P_a^{a+d}]$, which is a signed measure on
  $\sigma^d$. Decompose $R$ into positive and negative parts as
  $R=Q^+-Q^-$ and similarly for $R^d=Q^{+d}-Q^{-d}$. Notice that since
  $R^d$ is constructed using the marginals of $\P$, then $R(E) =
  R^d(E)$ for all $E\in\sigma^d$. Now since $R$ is
  the difference of probability measures, we must have that
  \begin{align}
    0 &=R(\Omega)=Q^+(\Omega)-Q^-(\Omega)\nonumber\\
    &=Q^+(D)+Q^+(D^c) -Q^-(D)-Q^-(D^c) \label{eq:19}
  \end{align}
  for all $D\in \sigma$.

  Define $Q=Q^++Q^-$.  Let $\epsilon>0$. Let $C\in\sigma$ be such that
  \begin{equation}
    Q(C) =\beta(a)=Q^+(C)=Q^-(C^c).
    \label{eq:17}
  \end{equation}
  Such a set $C$ is guaranteed by the Hahn
  decomposition theorem (letting $C^*$ be a set which attains the
  supremum in~(\ref{eq:11}), we can throw away any subsets with negative
  $R$ measure) and~(\ref{eq:19}) assuming without loss of
  generality that $\P_a(C)>[\P_{-\infty}^0\otimes \P_a^\infty](C)$. We can use the field
  $\sigma_f=\bigcup_d \sigma^d$ to
  approximate $\sigma$ in the sense that, for all $\epsilon$, we can
  find $A\in\sigma_f$ such that $Q(A\Delta C) < \epsilon/2$ (see
  Theorem D in~\citet[\S 13]{Halmos1974} or Lemma
  A.24 in~\citet{Schervish1995}). Now,
  \begin{align*}
    Q(A\Delta C) &= Q(A\cap C^c)+Q(C\cap A^c)\\
    &= Q^-(A \cap C^c) +Q^+(C\cap A^c)
  \end{align*}
  by~(\ref{eq:17}) since $A\cap C^c \subseteq C^c$ and $C\cap A^c
  \subseteq C$. Therefore, since $Q(A\Delta C) < \epsilon/2$, we have
  \begin{align}
    \label{eq:20}
    Q^-(A\cap C^c) &\leq \epsilon/2\\
    Q^+(A^c \cap C) & \leq \epsilon/2.\nonumber
  \end{align}
  Also,
  \begin{align*}
    Q(C) &= Q(A\cap C) + Q(A^c\cap C)\\ 
    &= Q^+(A\cap C) + Q^+(A^c\cap C)\\
    &\leq Q^+(A) +\epsilon/2
  \end{align*}
  since $A\cap C$ and $A^c\cap C$ are contained in $C$ and $A\cap C
  \subseteq A$. Therefore
  \[
  Q^+(A) \geq Q(C) - \epsilon/2.
  \]
  Similarly,
  \[
  Q^-(A) = Q^-(A\cap C) + Q^-(A\cap C^c) \leq 0 + \epsilon/2=\epsilon/2
  \]
  since $A\cap C \subseteq C$ and $Q^-(C)=0$ by~(\ref{eq:20}). Finally,
  \begin{align*}
    Q^{+d}(A) &\geq Q^{+d}(A) - Q^{-d}(A) = R^d(A)\\
    &= R(A) = Q^+(A) - Q^-(A)\\
    &\geq Q(C) - \epsilon/2-\epsilon/2 = Q(C)-\epsilon\\
    &= \beta(a) - \epsilon.
  \end{align*}
  And since $\beta^d(a) \geq Q^{+d}(A)$, we have that for all
  $\epsilon>0$ there exists $d$ such that for all $d_1>d$,
  \begin{align*}
    \beta^{d_1}(a) &\geq \beta^d(a) \geq Q^{+d}(A)\\
    &\geq \beta(a)-\epsilon.
  \end{align*}
  Thus, we must have that $L=\beta(a)$, so that
  $\beta^d(a)\rightarrow\beta(a)$ as desired.
\end{proof}

\begin{proof}[Proof of Theorem~\ref{thm:two}]
By the triangle inequality,
  \[
  |\widehat{\beta}^{d_n}(a) - \beta(a)| \leq
  |\widehat{\beta}^{d_n}(a)-\beta^{d_n}(a)|   + |\beta^{d_n}(a) - \beta(a)|.
  \]
  The first term on the right
  is bounded by the result in Theorem~\ref{thm:main}, where we have
  shown that $d_n=O(\exp\{W(\log n)\})$ is slow enough for the histogram
  estimator to remain consistent. 
That $\beta^{d_n}(a)
  \xrightarrow{d_n\rightarrow\infty} \beta(a)$ follows from
Lemma~\ref{lem:two2}. 
\end{proof}

\section{Discussion}
\label{sec:discussion}

We have shown that our estimator of the $\beta$-mixing coefficients is
consistent for the true coefficients $\beta(a)$ under some conditions on the
data generating process. There are numerous results in the statistics
and machine learning literatures which assume knowledge of the
$\beta$-mixing coefficients, yet as far as we know, this is the first estimator for
them. An ability to estimate these coefficients will allow researchers
to apply existing results to dependent data without the need to arbitrarily
assume their values. Despite the obvious utility of this estimator, as
a consequence of its novelty, it comes with 
a number of potential extensions which warrant careful exploration as well as
some drawbacks.

The reader will note that Theorem~\ref{thm:two} does not provide a
convergence rate.  The rate in Theorem~\ref{thm:main} applies only to
the difference between $\hat{\beta}^d(a)$ and $\beta^d(a)$.
In order to provide a rate in Theorem~\ref{thm:two}, we
would need a better understanding of the non-stochastic convergence of
$\beta^d(a)$ to $\beta(a)$. It is not immediately clear that
this quantity can converge at any well-defined rate. In particular, it
seems likely that the rate of convergence depends on the tail of the
sequence $\{\beta(a)\}_{a=1}^\infty$. 

Several other mixing and weak-dependence coefficients also have a
total-variation flavor, perhaps most notably $\alpha$-mixing
\cite{Doukhan1994,dedecker2007weak,Bradley2005}.  None of them have estimators,
and the same trick might well work for them, too.

The use of histograms rather than kernel density estimators for the joint and
marginal densities is somewhat surprising and not entirely necessary. As
mentioned above, \citet{Tran1989} proved that KDEs are consistent for
estimating the stationary density of a time series with $\beta$-mixing inputs,
so one {\em could} just replace the histograms in our esitmator with
KDEs. However, KDEs suffer from two major issues. Theoretically, we
need an analogue of the double asymptotic results proven for
histograms in Lemma~\ref{lem:three}. In particular, we need to
estimate increasingly higher dimensional densities as
$n\rightarrow\infty$. This does not cause a problem of
small-$n$-large-$d$ since $d$ is chosen as a function of $n$, however
it will lead to 
increasingly higher dimensional integration. For histograms, the integral is
always trivial, but in the case of KDEs, 
the numerical accuracy of the integration algorithm becomes
increasingly important. This issue could
swamp any efficiency gains obtained through the use of
kernels. However, this question certainly warrants further investigation.

The main drawback of an estimator based on a density estimate is its complexity. The mixing coefficients
are functionals of the joint and marginal distributions derived from the
stochastic process $\mathbf{X}$, however, it is unsatisfying to estimate
densities and solve integrals in order to estimate a single number. Vapnik's
main principle for solving problems using a restricted amount of information is
\begin{quote}
  When solving a given problem, try to avoid solving a more general
  problem as an intermediate step~\citep[p.~30]{Vapnik2000}.
\end{quote}
This principle is clearly violated here, but perhaps our seed will precipitate
a more aesthetically pleasing solution.

\bibliography{betamix.bib}
\end{document}